\documentclass{article}
\usepackage[utf8]{inputenc}

\usepackage{algorithm}
\usepackage{algorithmic}

\usepackage[table]{xcolor}
\usepackage{tabularx}
\usepackage{booktabs}

\usepackage{bbold}
\usepackage{amsmath,amsthm,amssymb}
\usepackage{amsfonts}
\usepackage{nicefrac} 
\usepackage{cases}
\usepackage{bm}

\usepackage[utf8]{inputenc}
\usepackage[T1]{fontenc}
\usepackage{pifont}
\usepackage{xspace}
\usepackage{hyphenat} 
\usepackage{hyperref}       
\usepackage{url}            
\usepackage{nicefrac}       
\usepackage{microtype}      
\usepackage{thmtools,thm-restate}
\usepackage{graphicx}
\usepackage{subcaption} 
\usepackage{cleveref}
\usepackage{color, soul} 	
\usepackage{tabularx}
\usepackage{booktabs}
\usepackage{multicol}

\usepackage{thm-restate}

\newtheorem*{theorem*}{Theorem}
\newtheorem{theorem}{Theorem}
\newtheorem{lemma}{Lemma}

\definecolor{myRed}{RGB}{228, 26, 28}
\definecolor{myBlue}{RGB}{55, 126, 184}
\definecolor{myGreen}{RGB}{77, 175, 74}
\definecolor{myPurple}{RGB}{152, 78, 163}
\definecolor{myOrange}{RGB}{255, 127, 0}
\definecolor{myGrey}{RGB}{153, 153, 153}
\definecolor{myBrown}{RGB}{166, 86, 40}
\definecolor{myPink}{RGB}{247, 129, 191}

\newcommand{\R}{\mathbb{R}}

\renewcommand{\algorithmicrequire}{\textbf{Input:}}
\renewcommand{\algorithmicensure}{\textbf{Output:}}
\renewcommand{\algorithmicforall}{\textbf{for each}}

\renewcommand{\algorithmiccomment}[1]{#1}

\usepackage[bottom]{footmisc}

\PassOptionsToPackage{numbers,sort}{natbib}

\usepackage[final]{neurips_2020}
\title{What if Neural Networks had SVDs?}

\usepackage{dblfloatfix} 
\usepackage{wrapfig}

\author{
Alexander Mathiasen\thanks{Aarhus University, \{alexander.mathiasen, fhvilshoj, mrjakobdk\}@gmail.com, davide@cs.au.dk} 
\And
Frederik Hvilsh\o j\footnotemark[1]
\And
Jakob R\o dsgaard J\o rgensen\footnotemark[1]
\And
Anshul Nasery\footnotemark[1] \; \thanks{Indian Institute of Technology, Bombay, anshulnasery@gmail.com}
\And
Davide Mottin\footnotemark[1]
}

\begin{document}

\maketitle

\begin{abstract}
Various Neural Networks employ time-consuming matrix operations like matrix inversion. 
Many such matrix operations are faster to compute given the Singular Value Decomposition (SVD). 
Techniques from \cite{svdnn,orthhh} allow using the SVD in Neural Networks without computing it. 
In theory, the techniques can speed up matrix operations, however, in practice, they are not fast enough.  
We present an algorithm that is fast enough to speed up several matrix operations.
The algorithm increases the degree of parallelism of an underlying matrix multiplication $H\cdot X$ where $H$ is an orthogonal matrix represented by a product of Householder matrices. 
\end{abstract}

\section{Introduction} \label{sec:intro}
\begin{wrapfigure}[12]{r}{0.3975\textwidth}
    \centering
    \vspace{-7mm}
    \includegraphics[width=0.41\textwidth]{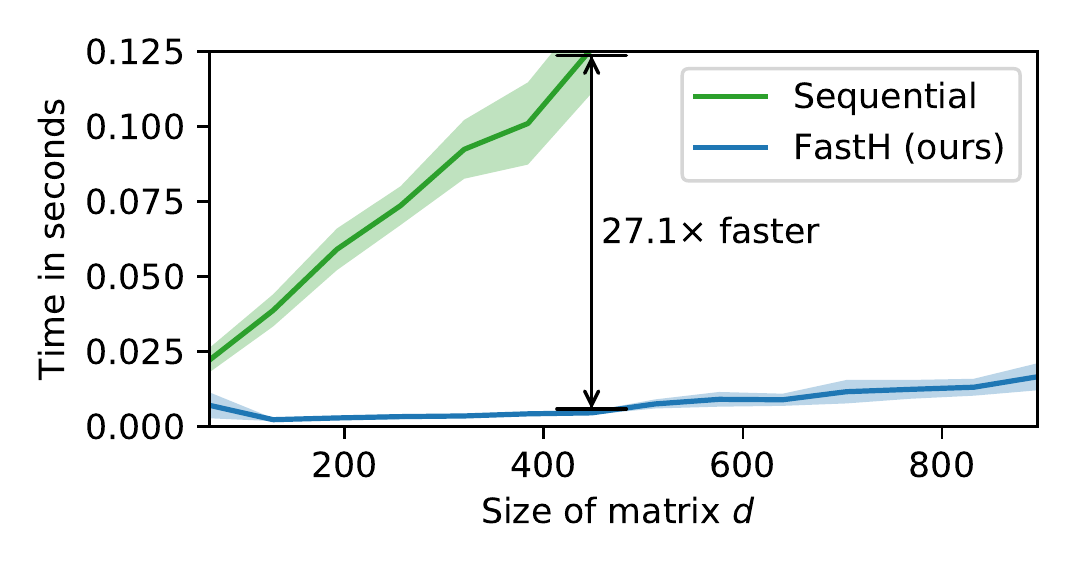}
    \caption{Time consumption of matrix inversion in Neural Networks. 
    The plot compares FastH against the sequential algorithm from \cite{svdnn} (see \Cref{sec:exp}). 
    }
    \label{fig:time}
\end{wrapfigure}
What could be done if the Singular Value Decomposition (SVD) of the weights in a Neural Network was given?  
Time-consuming matrix operations, such as matrix inversion \cite{emerging}, could be computed faster, reducing training time. 
However, on $d\times d$ weight matrices it takes $O(d^3)$ time to compute the SVD, which is not faster than computing the matrix inverse in $O(d^3)$ time. 
In Neural Networks, one can circumvent the SVD computation by using the SVD reparameterization from \cite{svdnn}, which, in theory, reduces the time complexity of matrix inversion from $O(d^3)$ to $O(d^2)$. 
However, in practice, the SVD reparameterization attains no speed-up for matrix inversion on GPUs. 

The difference between theory and practice occurs because the previous technique increase sequential work, which is not taken into account by the time complexity analysis. 
On a $d\times d$ weight matrix, the previous technique entails the computation of $O(d)$ sequential inner products, which is ill-fit for parallel hardware like a GPU because the GPU cannot utilize all its cores.
For example, if a GPU has 4000 cores and computes sequential inner products on 100-dimensional vectors, it can only utilize 100 cores simultaneously, leaving the remaining 3900 cores to run idle.


We introduce a novel algorithm, FastH, which increases core utilization, leaving less cores to run~idle.
This is accomplished by increasing the degree of parallelization of an underlying matrix multiplication $H\cdot X$ where $H$ is an orthogonal matrix represented by a product of Householder matrices.
FastH retains the same desirable time complexity as the sequential algorithm from \cite{svdnn} while reducing the number of sequential operations. 
On a mini-batch of size $m>1$, FastH performs $O(d/m+m)$ sequential matrix-matrix operations instead of $O(d)$ sequential vector-vector operations.

In practice, FastH is faster than all algorithms from \cite{svdnn}, e.g., FastH is $27$ times faster than their sequential algorithm, see \Cref{fig:time}. Code \url{www.github.com/AlexanderMath/fasth}. 



\section{Background} \label{sec:background}
\subsection{Fast Matrix Operations Using SVD }
The SVD allows faster computation of many matrix operations commonly used by Neural Networks. 
A few examples include matrix determinant \cite{nice}, matrix inverse \cite{glow}, Spectral Normalization \cite{sngan}, the matrix exponential \cite{matrixexp1}, the Cayley transform \cite{cayley1}, weight decay, condition number and compression by low-rank approximation \cite{compresswithsvd}. 
Proofs can be found in most linear algebra textbooks, see, e.g., \cite{strang}. 

\subsection{The SVD Reparameterization}\label{subsec:svdtrick} 
This subsection describes how \cite{svdnn} allows for using the SVD of the weight matrices in Neural Networks without computing them, and in particular, how this approach is limited by the computation of sequential inner products. 
Let $W=U\Sigma V^T$ be the SVD of a weight matrix $W$ where $\Sigma$ is a diagonal matrix and $U,V$ are orthogonal matrices, i.e, $U^T=U^{-1}$ and $V^T=V^{-1}$. 
The goal is to perform gradient descent updates to $W$ while preserving the SVD.
Consider updating $U,\Sigma,V$ a small step $\eta\in \R$ in the direction of gradients $\nabla_U, \nabla_\Sigma, \nabla_V$. 
\begin{align*}
    \Sigma'=\Sigma - \eta \nabla_\Sigma, \quad 
    U' = U - \eta \nabla_U, \quad 
    V' = V - \eta \nabla_V.
\end{align*}
While $\Sigma'$ remains diagonal, both $U'$ and $V'$ are in general not orthogonal, which is needed to preserve the SVD. 
To this end, \cite{svdnn} suggested using a technique from \cite{orthhh} which decomposes an orthogonal matrix as a product of $d$ Householder matrices $H_1,\dots ,H_d$:
\begin{equation}\label{equ:hh}
U=\prod_{i=1}^{d} H_i \quad \quad H_i=I-2\frac{v_iv_i^T}{||v_i||_2^2}\quad\quad v_i\in \R^d.
\end{equation}
Householder matrices satisfy several useful properties. 
In particular, the matrix $U$ remains orthogonal under gradient descent updates $v_i=v_i-\eta\nabla_{v_i}$ \cite{orthhh}.
Furthermore, all products of Householder matrices are orthogonal, and any $d\times d$ orthogonal matrix can be decomposed as a product of $d$ Householder matrices \cite{qrhh}. 
Householder matrices thus allow us to perform gradient descent over orthogonal matrices, which allows us to preserve the SVD of $W$ during gradient descent updates. 

\paragraph{Multiplication. }
One potential issue remains. 
The Householder decomposition might increase the time it takes to multiply $UX$ for a mini-batch $X\in \R^{d \times m }$ during the forward pass. Computing 
    $UX=H_1 \cdots (H_{d-1}(H_d\cdot  X))$
takes $d$ Householder multiplications. 
If done sequentially, as indicated by the parenthesis, each Householder multiplication can be computed in $O(dm)$ time \cite{svdnn}. 
All $d$ multiplications can thus be done in $O(d^2m)$ time. 
Therefore, the Householder decomposition does not increase the time complexity of computing $UX$. 

Unfortunately, the $O(d^2m)$ time complexity comes at the cost of multiplying each Householder matrix sequentially,
and each Householder multiplication entails computing an inner product, see \Cref{equ:hh}. 
The multiplication $UX$ then requires the computation of $O(d)$ inner products sequentially. 
Such sequential computation is slow on parallel hardware like GPUs, much slower than normal matrix multiplication.
To exploit GPUs, \cite{svdnn} suggested using a parallel algorithm that takes $O(d^3)$ time, but this is no faster than computing the SVD. 

We are thus left with two options: (1) an $O(d^2m)$ sequential algorithm and (2) an $O(d^3)$ parallel algorithm.
The first option is undesirable since it entails the sequential computation of $O(d)$ inner products. 
The second option is also undesirable since it takes $O(d^3)$ which is the same as computing the SVD, i.e., we might as-well just compute the SVD. 
In practice, both algorithms usually achieve no speed-up for the matrix operations like matrix inversion as we show in \Cref{subsec:matrix_operations}. 

Our main contribution is a novel parallel algorithm, FastH, which resolves the issue with sequential inner products without increasing the time complexity. 
FastH takes $O(d^2m)$ time but performs $O(d/m+m)$ sequential matrix-matrix operations instead of $O(d)$ sequential vector-vector operations (inner products). 
In practice, FastH is up to $6.2\times $ faster than the parallel algorithm and up to $27.1\times$ faster than the sequential algorithm, see \Cref{subsec:running_time}. 

\paragraph{Mathematical Setting. } 
We compare the different methods by counting the number of sequential matrix-matrix and vector-vector operations. 
We count only once when other sequential operations can be done in parallel. 
For example, processing $v_1,...,v_{d/2}$ sequentially while, in parallel, processing $v_{d/2+1},...,v_d$ sequentially, we count $d/2$ sequential vector-vector operations. 

\paragraph{Orthogonal Gradient Descent. }
The SVD reparameterization performs gradient descent over orthogonal matrices. 
This is possible with Householder matrices, 
however, other techniques, such as \cite{matrixexp2,cayley2}, rely on the matrix exponential and the Cayley map, respectively. 
For $d\times d$ matrices both techniques spend $O(d^3)$ time, which is no faster than computing the SVD.

\section{Fast Householder Multiplication (FastH)} 
\label{sec:algo}

\subsection{Forward Pass}\label{subsec:algo_forward}
Our goal is to create an $O(d^2m)$ algorithm with few sequential operations that solves the following problem: 
Given an input $X\in\R^{d\times m}$ with $d>m>1$ and Householder matrices $H_1,...,H_d$, compute the output $A=H_1\cdots H_d X$. 
For simplicity, we assume $m$ divides $d$. 

Since each $H_i$ is a $d\times d$ matrix, it would take $O(d^3)$ time to read the input $H_1,...,H_d$. 
Therefore, we represent each Householder matrix $H_i$ by its Householder vector $v_i$ such that $H_i=I-2v_iv_i^T/||v_i||_2^2$. 
A simplified version of the forward pass of FastH proceeds as follows: 
divide the Householder product $H_1 \cdots H_d$ into smaller products $P_1\cdots P_{d/m}$ so each $P_i$ is a product of $m$ Householder matrices:
\begin{equation}\label{equ:P}
P_i = H_{(i-1)\cdot m + 1} \cdots H_{i \cdot m }\quad\quad i=1,...,d/m. 
\end{equation}
All $d/m$ products $P_i$ can be computed in parallel. 
The output can then be computed by $A=P_1\cdots P_{d/m} X$ instead of $A=H_1\cdots H_d X$, which reduces the number of sequential matrix multiplications from $d$ to $d/m$. 

This algorithm computes the correct $A$. However, the time complexity increases due to two issues. 
First, multiplying each product $P_i$ with $X$ takes $O(d^2m)$ time, a total of $O(d^3)$ time for all $d/m$ products. 
Second, we need to compute all $d/m$ products $P_1,...,P_{d/m}$ in $O(d^2m)$ time, so each product $P_i$ must be computed in $O(d^2m/(d/m))=O(dm^2)$ time. 
If we only use the Householder structure, it takes $O(d^2m)$ time to compute each $P_i$, which is not fast enough. 

Both issues can be resolved, yielding an $O(d^2m)$ algorithm. 
The key ingredient is a linear algebra result \cite{wydec} that dates back to 1987. The result is restated in \Cref{lemma:wydec}. 
\begin{lemma}\label{lemma:wydec}For any $m$ Householder matrices $H_1,...,H_m$ there exists $W,Y\in \R^{d\times m}$ st. 
$I-2WY^T = H_1 \cdots H_m$. 
Computing $W$ and $Y$ takes $O(dm^2)$ time and $m$ sequential Householder multiplications. 
\end{lemma}
For completeness, we provide pseudo-code in \Cref{algo:forward}. 
\Cref{thm:forward_algorithm} states properties of \Cref{algo:forward} and its proof clarify how \Cref{lemma:wydec} solves both issues outlined above.

\begin{theorem}\label{thm:forward_algorithm}
Algorithm 1 computes $H_1\cdots H_d X$ in $O(d^2m)$ time with $O(d/m + m)$ sequential matrix multiplications. 
\end{theorem}
\begin{proof}

\textbf{Correctness.} 
Each iteration of Step 2 in \Cref{algo:forward} utilizes \Cref{lemma:wydec} to compute $A_i = A_{i+1}-2W_i(Y_i^TA_{i+1}) =P_i A_{i+1}$. Therefore, at termination, $A_1=P_1\cdots P_{d/m}X$. In Step~1, we used \Cref{lemma:wydec} to compute the $P_i$'s such that $A=H_1\cdots H_dX$ as wanted. 

\textbf{Time Complexity. } 
Consider the for loop in Step 1. 
By \Cref{lemma:wydec}, each iteration takes $O(dm^2)$ time. 
Therefore, the total time of the $d/m$ iterations is $O(dm^2d/m)=O(d^2m)$. 
Consider iteration $i$ of the loop in Step 2. 
The time of the iteration is asymptotically dominated by both matrix multiplications.  
Since $A_{i+1},X_i$ and $Y_i$ are $d\times m$ matrices, it takes $O(dm^2)$ time to compute both matrix multiplications. 
There are $d/m$ iterations so the total time becomes $O(dm^2 d/m)=O(d^2m)$. 

\textbf{Number of Sequential Operations.} 
Each iteration in Step 2 performs two sequential matrix multiplications. 
There are $d/m$ sequential iterations which gives a total of $O(d/m)$ sequential matrix multiplications. 
Each iteration in Step 1 performs $m$ sequential Householder multiplications to construct $P_i$, see \Cref{lemma:wydec}.
Since each iteration is run in parallel, the algorithm performs no more than $O(d/m+m)$ sequential matrix multiplications. 
\end{proof}

\paragraph{Remark. }
\Cref{subsec:algo_backward} extends the techniques from this section to handle gradient computations. 
For simplicity, this section had \Cref{algo:forward} compute only $A_1$, however, 
in \Cref{subsec:algo_backward} it will be convenient to assume $A_1,...,A_{d/m}$ are precomputed. 
Each $A_i=P_i\cdots P_{d/m}X$ can be saved during Step 2 of \Cref{algo:forward} without increasing asymptotic memory consumption. 

\begin{figure*}[b!]
    \centering
    \begin{subfigure}{0.48\textwidth}
          \centering
          \includegraphics[width=\textwidth,page=7]{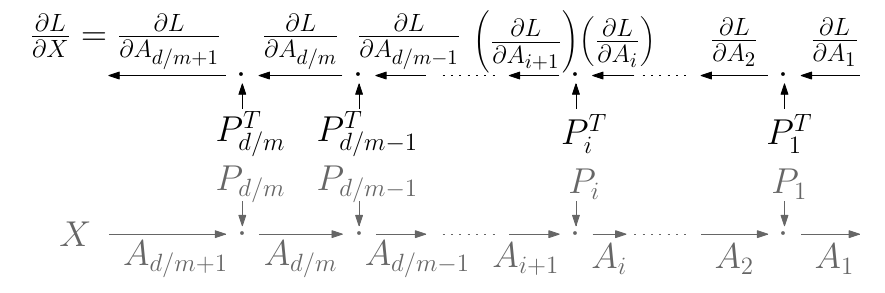}
          \caption{Step 1: Sequential part of \Cref{algo:backward}. }
          \label{subfig:sequential}
    \end{subfigure}
    \begin{subfigure}{0.48\textwidth}
          \centering
          \includegraphics[width=\textwidth,page=11]{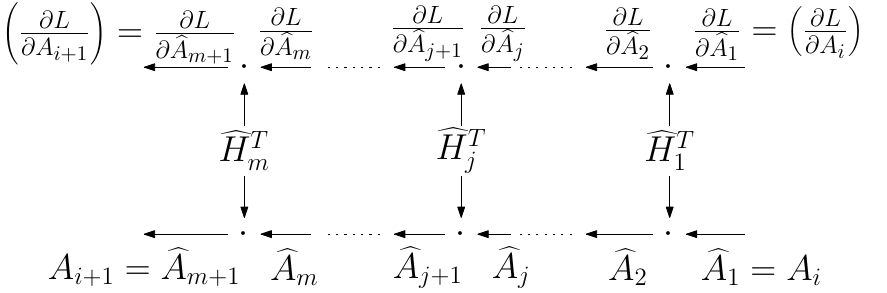}
          \caption{Step 2: The $i$'th subproblem in \Cref{algo:backward}. }
          \label{subfig:parallel}
    \end{subfigure}
    \caption{Computational graph of Step 1 and the $i$'th subproblem in Step 2 from \Cref{algo:backward}. 
    }
    \label{fig:schematic}
\end{figure*}

\subsection{Backwards Propagation } 
\label{subsec:algo_backward} 
This subsection extends the techniques from \Cref{subsec:algo_forward} to handle gradient computations. 
Our goal is to create an $O(d^2m)$ algorithm with few sequential operations that solves the following problem: 
Given $A_1,\dots, A_{d/m+1}$, 
$P_1,...,P_{d/m}$ and 
$\frac{\partial L }{ \partial A_1}$ for some loss function $L$, 
compute $\frac{\partial L}{\partial X}$ and 
$\frac{\partial L}{\partial v_1},...,\frac{\partial L}{\partial v_d}$, where $v_j$ is a Householder vector st. $H_j=I-2v_jv_j^T/||v_j||^2_2$. 

Since each $P_i$ is a $d\times d$ matrix, it would take $O(d^3/m)$ time to read the input $P_1,...,P_{d/m}$. 
Therefore, we represent each $P_i$ by its WY decomposition $P_i=I-2WY^T$. 

On a high-level the backward pass of FastH has two steps. 
\paragraph{Step 1. } 
Sequentially compute $\frac{\partial L}{\partial A_2}$, $\frac{\partial L}{\partial A_3},...,\frac{\partial L}{\partial A_{d/m+1}}$ by
\begin{equation}\label{equ:back_A}
\frac{\partial L}{\partial A_{i+1}}=\left[ \frac{\partial A_i}{\partial A_{i+1}} \right]^T \frac{\partial L}{\partial A_i}= P_i^T \frac{\partial L}{\partial A_i}
\end{equation}
This also gives the gradient wrt. $X$ since  $X=A_{d/m+1}$. 

\paragraph{Step 2.}
Use $\frac{\partial L}{\partial A_1},...,\frac{\partial L}{\partial A_{d/m}}$ from Step 1 to compute the gradient $\frac{\partial L}{\partial v_j}$ for all $j$. 
This problem can be split into $d/m$ subproblems, which can be solved in parallel, one subproblem for each $\frac{\partial L}{\partial A_i}$. 

\paragraph{Details.} For completeness, we state pseudo-code in \Cref{algo:backward}, which we now explain with the help of  \Cref{fig:schematic}. 
\Cref{subfig:sequential} depicts a computational graph of Step 1 in \Cref{algo:backward}. 
In the figure, consider $\frac{\partial L}{\partial A_1}$ and $P_1^T$, which both have directed edges to a multiplication node (denoted by~$\cdot$).
The output of this multiplication is $\frac{\partial L}{\partial A_2}$ by \Cref{equ:back_A}. 
This can be repeated to obtain $\frac{\partial L}{\partial A_{2}},...,\frac{\partial L}{\partial A_{d/m+1}}$. 

Step 2 computes the gradient of all Householder vectors $\frac{\partial L}{\partial v_j}$. 
This computation is split into $d/m$ distinct subproblems that can be solved in parallel. 
Each subproblem concerns $\frac{\partial L}{\partial A_i}$ and the product $P_i$, see line 8-10 in \Cref{algo:backward}. 

To ease notation, we index the Householder matrices of $P_i$ by $P_i=\widehat H_1 \cdots \widehat H_m$.
Furthermore, we let $\widehat A_{m+1}:=A_{i+1}$ and $\widehat A_j := \widehat H_j \widehat A_{j+1} $.
The notation implies that $\widehat A_1 = \widehat H_1\cdots \widehat H_{m}\widehat A_{m+1}=P_i A_{i+1}=A_i$. 
The goal of each subproblem is to compute gradients wrt. the Householder vectors $\widehat v_m, ..., \widehat v_1$ of $\widehat H_m,..., \widehat H_1$. 
To compute the gradient of $\widehat v_j$, we need $\widehat A_{j+1}$ and $\frac{\partial L}{\partial \widehat A_j}$, which can be computed by: 
\begin{align}\label{equ:A_hat}
\widehat A_{j+1} = \widehat H_j^{-1} \widehat A_j = \widehat H_{j}^T \widehat A_j 
\quad \quad \quad 
\frac{\partial L}{\partial \widehat{A}_{j+1}} = \left[\frac{\partial\widehat{A}_j} {\partial \widehat{A}_{j+1}} \right]^T \frac{\partial L}{\partial \widehat{A}_j} = \widehat H_j^T \frac{\partial L}{\partial \widehat{A}_j}
\end{align}
\Cref{subfig:parallel} depicts how 
$\widehat A_2,...,\widehat A_{m+1}$ and 
$\frac{ \partial L}
{\partial \widehat A_2}
,...,
\frac{\partial L}{\partial \widehat A_{m+1}}$ 
can be computed given $\widehat A_1$ and $\frac{ \partial L}{\partial \widehat A_1}$.
Given $\widehat A_{j+1}$ and $\frac{\partial L}{\partial \widehat A_j}$, we can compute $\frac{\partial L}{\partial \widehat v_j}$ as done in \cite{svdnn, orthhh}. 
For completeness, we restate the needed equation in our notation, see \Cref{equ:back_v}. 

Let $a^{(l)}$ be the $l$'th column of $\widehat A_{j+1}$ 
and let $g^{(l)}$ be the $l$'th column of  $\frac{\partial L}{\partial \widehat A_j}$. 
The sum of the gradient over a mini-batch of size $m$ is then: 
\begin{equation}\label{equ:back_v}
-\frac{2}{||\widehat v_j||^2_2} 
\sum_{l=1}^m
(\widehat v_j^T a^{(l)})g^{(l)} 
+ (\widehat v_j^T g^{(l)}) a^{(l)} 
- \frac{2}{||\widehat v_j||^2_2} (\widehat v_j^T a^{(l)}) (\widehat v_j^T g^{(l)}) \widehat v_j 
\end{equation}

\Cref{thm:backwards} states properties of \Cref{algo:backward}.

\begin{theorem}\label{thm:backwards}
\Cref{algo:backward} computes $\frac{\partial L}{\partial X}$ and 
$\frac{\partial L}{\partial v_1},...,\frac{\partial L}{\partial v_d}$
in $O(d^2m)$ time with $O(d/m+m)$ sequential matrix~multiplications. 
\end{theorem}
\begin{proof} See the Supplementary Material \ref{sup:proof_backwards}.  \end{proof}

\begin{minipage}{0.44\textwidth}
\begin{algorithm}[H]
    \caption{FastH Forward}
    \label{algo:forward}
    \footnotesize
    \begin{algorithmic}[1]
       \STATE {\bfseries Input:} $X\in \R^{d\times m}$ and $H_1, ..., H_d\in\R^{d\times d}$. 
       \vspace{-1.5mm} 
       \STATE {\bfseries Output:} $A_{1}=H_1 \cdots H_d  X$. 
       \vspace{1.5mm} 
       \STATE // Step 1 
       \FOR [\textbf{in parallel}] {$i=d/m$ {\bfseries to} $1$}
       \STATE {Compute $Y_i,W_i\in \R^{d\times m}$ st. }\\
        {$\hspace{5mm}P_i=I-2W_iY_i^T$} \COMMENT{\hfill $\triangleright \; O(dm^2)$} \\
       by using \Cref{lemma:wydec}. 
       \ENDFOR
       \vspace{1.5mm}
       \STATE // Step 2 
       \STATE $A_{d/m+1}=X.$
       \FOR[\textbf{sequentially}] {$i=d/m$ {\bfseries to} $1$}
       \STATE $A_{i} = A_{i+1} - 2W_i(Y_i^TA_{i+1})$. \COMMENT{\hfill $\triangleright \;O(dm^2)$}
       \ENDFOR
       \STATE \textbf{return} $A_1$. 
       \vspace{16.5mm}
    \end{algorithmic}
\end{algorithm}
\end{minipage}
\hfill
\begin{minipage}{0.54\textwidth}
\begin{algorithm}[H]
    \caption{FastH Backward}
    \label{algo:backward}
    \footnotesize
    \begin{algorithmic}[1]
       \STATE {\bfseries Input:}  $A_1,...,A_{d/m+1}$, $P_1,..., P_{d/m}$ and $\frac{\partial L}{\partial A_1}$. 
       \STATE {\bfseries Output:} 
       $\frac{\partial L}{\partial X}$ and 
       $\frac{\partial L}{\partial v_k}$ for all $k$  where 
       $H_k=I-2\frac{v_kv_k^T}{||v_k||^2_2}$.
       \STATE // Step 1 
       \FOR[\textbf{sequentially}]{$i=1$ {\bfseries to} $d/m$}
       \STATE $\frac{ \partial L}{\partial A_{i+1}}=P_i^T\frac{\partial L}{\partial A_i}$ \cref{equ:back_A}.  \COMMENT{\hfill $\triangleright \; O(dm^2)$}
       \ENDFOR
       \vspace{2mm}
       \STATE // Step 2 
       \FOR[\textbf{in parallel}]{$i=1$ {\bfseries to} $d/m$}
    	\STATE Let $\frac{\partial L}{\partial \widehat{A}_1}=\left(\frac{\partial L} {\partial A_i}\right)$. 
        \STATE To ease notation, let $P_i= \widehat H_1 \cdots \widehat H_m$.
       \FOR{$j=1$ {\bfseries to} $m$}
    	\STATE Compute $\widehat A_{j+1}, \frac{\partial L}{\partial \widehat A_{j}}$
    	see \cref{equ:A_hat}.   \COMMENT{\hfill $\triangleright \; O(dm)$} \\
    	\STATE Compute $\frac{\partial L}{\partial \widehat v_j}$ 
    	using $\widehat A_{j+1},\frac{\partial L}{\partial \widehat A_j}$, 
    	\cref{equ:back_v}. 
    	\COMMENT{\hfill $\triangleright \; O(dm)$}
       \ENDFOR
       \ENDFOR
       \STATE \textbf{return} $\frac{\partial L}{\partial X} =\frac{\partial L}{\partial A_{d/m+1}}$ and $\frac{\partial L}{\partial v_k}$ 
       for all $k=1,...,d$. 
    \end{algorithmic}
\end{algorithm}
\end{minipage}

\subsection{Extensions} 
\paragraph{Trade-off. } 
Both Algorithm 1 and Algorithm 2 can be extended to take a parameter $k$ that controls a trade-off between \emph{total time complexity} and \emph{the amount of parallelism}. 
This can be achieved by changing the number of Householder matrices in each product $P_i$ from the mini-batch size $m$ to an integer $k$. 
The resulting algorithms take $O(d^2k+d^2m)$ time, $O(d^2m/k)$ space and has $O(d/k+k)$ sequential matrix multiplications.
This extension has the practical benefit that one can try different values of $k$ and choose the one that yields superior performance on a particular hardware setup.
Note that we never need to search for $k$ more than one time. 
The number of sequential matrix multiplications $O(d/k+k)$ is minimized when $k=O(\sqrt{d})$. 
For a constant $c>1$, we can find the best $k\in\{2,3,..., c\lceil \sqrt{d}\rceil \}$ by trying all $O(\sqrt{d})$ values. 
The search needs to be done only once and takes $O(\sqrt{d}(d^2k+d^2m))=O(d^3+d^{2.5}m)$ time. 
In practice, the time it took to find $k$ was negligable, e.g., on the hardware we describe in \Cref{sec:exp} we found $k$ in less than $1s$ for $d=784$. 

\paragraph{Rectangular Matrices. } We can use the SVD reparametrization for rectangular $W\in\mathbb{R}^{n\times m}$. 
Use orthogonal $U\in\mathbb{R}^{n \times n},V\in\mathbb{R}^{m\times m}$ and diagonal $\Sigma\in\mathbb{R}^{n\times m}$ and let $W=U\Sigma V^T$. 

\paragraph{Convolutional Layers. }
So far, we have considered the SVD reparameterization for matrices which corresponds to fully connected layers. 
The matrix case extends to convolutions by, e.g., $1\times 1$ convolutions \cite{glow}. 
The SVD reparameterization can be used for such convolutions without increasing the time complexity. 
On an input with height $h$ and width $w$ FastH performs $O(d/m + mhw)$ sequential matrix multiplications instead of the $O(d)$ sequential inner products. 

\paragraph{Recurrent Layers. }
The SVD reparameterization was developed for Recurrent Neural Networks (RNNs) \cite{svdnn}. 
Let $r$ be the number of recurrent applications of the RNN. 
FastH performs $O(d/m+rm)$ sequential matrix operations instead of the $O(d)$ sequential inner products.

\section{Experiments} 
\label{sec:exp}
This section contains two experiments. 
\Cref{subsec:running_time} compares the running time of FastH against alternatives.  
\Cref{subsec:matrix_operations} shows that FastH speeds up matrix operations. 
To simulate a realistic machine learning environment, we performed all experiments on a standard machine learning server using a single NVIDIA RTX 2080 Ti. 

\subsection{Comparing Running Time }
\label{subsec:running_time}
This subsection compares the running time of FastH against four alternative algorithms. 
We compare the time all algorithms spend on gradient descent with a single orthogonal matrix, since 
such constrained gradient descent dominates the running time of the SVD reparameterization. 

We first compare FastH against the parallel and sequential algorithm from \cite{svdnn}, all three algorithms rely on the Householder decomposition. 
For completeness, we also compare against approaches that does not rely on the Householder decomposition, in particular, the matrix exponential and the Cayley map \cite{matrixexp2}\footnote{
For the matrix exponential and the Cayley map we used the open-source implementation
\href {https://github.com/Lezcano/expRNN } {https://github.com/Lezcano/expRNN }  
from \cite{matrixexp2}.  
For the Householder decomposition, we used the open-source implementation 
\href {https://github.com/zhangjiong724/spectral-RNN}
{https://github.com/zhangjiong724/spectral-RNN}
of the sequential and parallel algorithm from~\cite{svdnn}. 
}. 
See Supplementary Material \ref{sup:running_time} for further details.

We measure the time of a gradient descent step with a weight matrix $W\in \R^{d\times d}$ and a mini-batch $X\in\R^{d\times m}$, where $m=32$ and $d=1\cdot 64,2\cdot 64,...,48\cdot 64$. 
We ran each algorithm $100$ times, and we report mean time $\mu$ with error bars $[\mu-\sigma, \mu+\sigma]$ where $\sigma$ is the standard deviation of running time over the $100$ repetitions.

\Cref{fig:run} depicts the running time on the y-axis, as the size of the $d\times d$ matrices increases on the x-axis. 
For $d>64$, FastH is faster than all previous approaches. 
At $d=64$ FastH is faster than all previous approaches, except the parallel algorithm. 
Previous work employ sizes $d=192$ in \cite{glow} and $d=784$ in \cite{svdnn}. 

\begin{figure}[t]
    \begin{subfigure}{0.49\textwidth}
        \centering
        \includegraphics[width=0.98\textwidth]{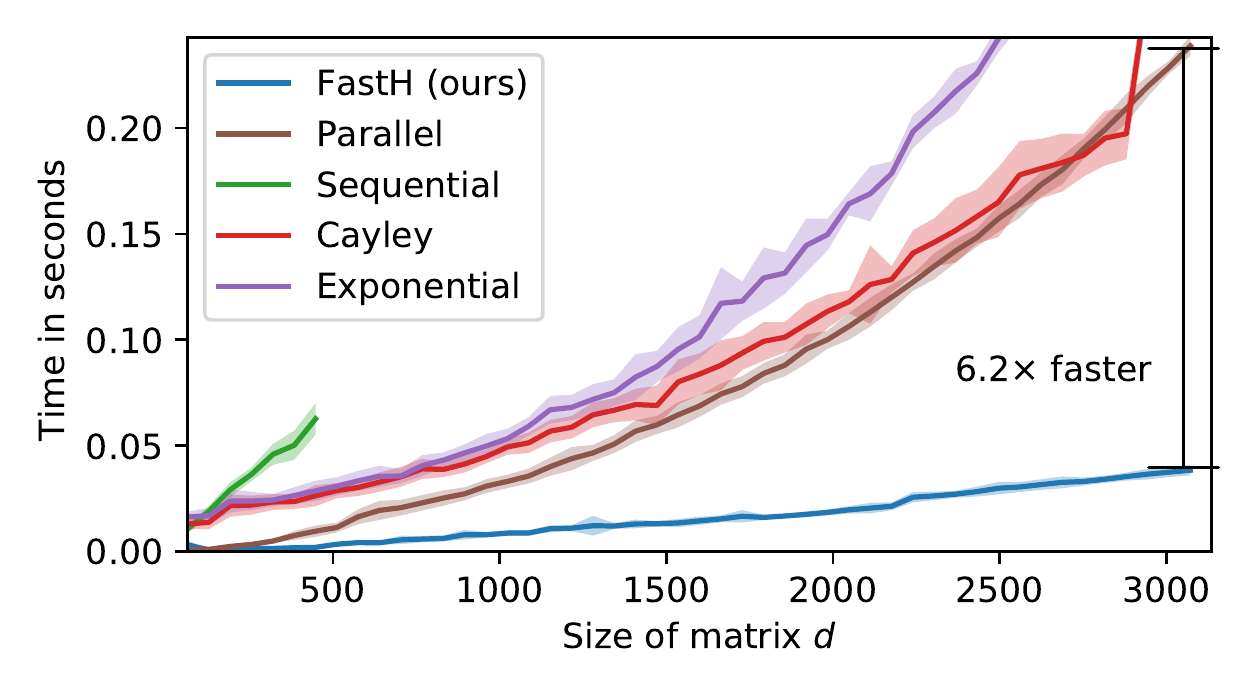}
        \caption{Running time.}
        \label{fig:run}
    \end{subfigure}
    \hspace{0.02\textwidth}
    \begin{subfigure}{0.49\textwidth}
        \centering
        \includegraphics[width=0.98\textwidth]{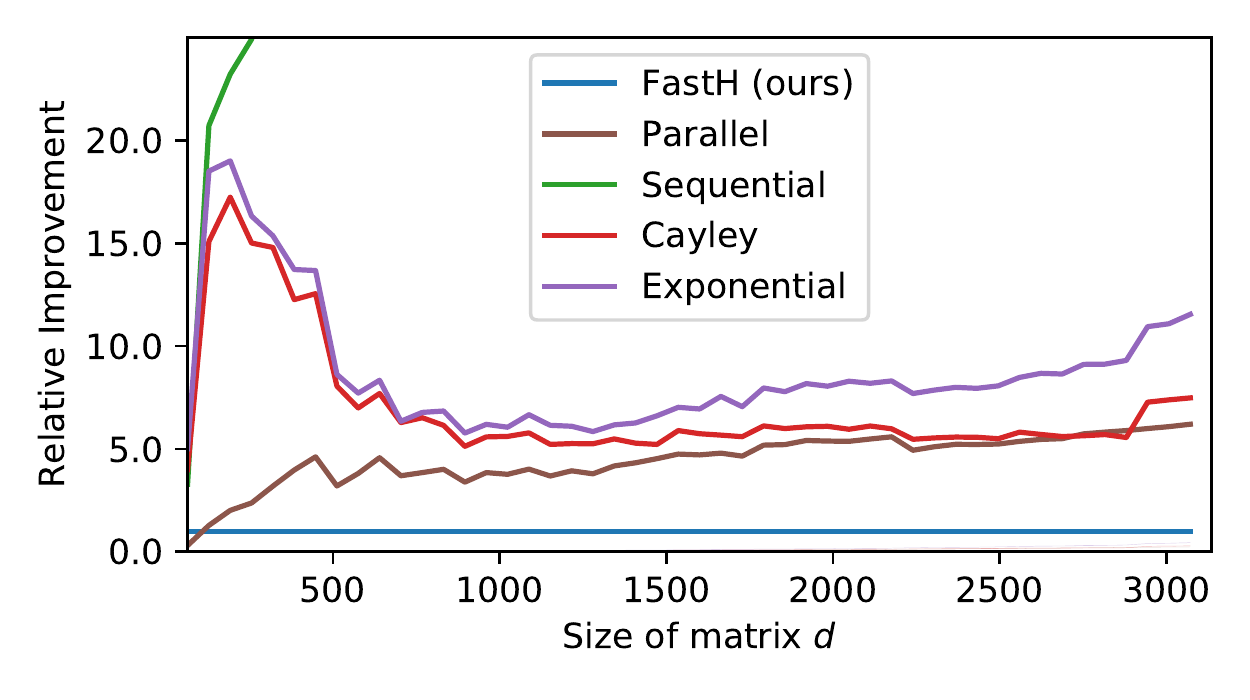}
        \caption{Relative improvement.}
        \label{fig:relative}
    \end{subfigure}
    \caption{Comparisons of the running times for FastH against previous algorithms. 
    The sequential algorithm from \cite{svdnn} crashed when $d>448$. (a)
        Running times of different algorithms for $d\times d$ matrices.
    (b)
        Running times of FastH relative to previous algorithms, i.e., the mean time of a previous algorithm divided by the mean time of FastH. 
    }
    \label{fig:runningtimes}
\end{figure} 

\Cref{fig:relative} depicts how much faster FastH is relative to the previous algorithms, i.e., the mean time of a previous algorithm divided by the time of FastH, which we refer to as relative improvement.  
For $d>500$, the relative improvement of FastH increases with $d$. 

At $d=448$ FastH is roughly $25\times$ faster than the sequential algorithm. 
FastH is even faster with $d=3072$ than the sequential algorithm with $d=448$. 
Previous work like \cite{emerging,sylvesternf} use the Householder decomposition with the sequential algorithm.
Since FastH computes the same thing as the sequential algorithm, it can be used to reduce computation time with no downside.

\begin{table}[h]
\centering
\caption{Relating standard method to matrix decompositions for computing matrix operations. }
\label{tab:overview}
\begin{tabular}{lll}
\toprule
Matrix Operation & Standard Method & SVD or Eigendecomposition  \\ \midrule
Determinant             & \textsc{torch.slogdet}(W)                & $\sum_{i=1}^d \lg |\Sigma_{ii}|$       \\
Inverse                 & \textsc{torch.inverse}(W)                & $V \Sigma^{-1} U^T$               \\
Matrix Exponential      & Padé Approximation \cite{matrixexp2}  & $U e^\Sigma U^T $                 \\
Cayley map              & \textsc{torch.solve}(I-W, I+W)        & $U(I{-}\Sigma)(I{+}\Sigma)^{-1}U^T$   \\
\bottomrule
\end{tabular}
\end{table}

%
%

\subsection{Using the SVD to Compute Matrix Operations }\label{subsec:matrix_operations}
This subsection investigates whether the SVD reparameterization achieves practical speed-ups for matrix operations like matrix inversion.  
We consider four different matrix operations. 
For each operation, we compare the SVD reparameterization against the standard method for computing the specific matrix operation, see \Cref{tab:overview}. 
\begin{wrapfigure}[16]{R}{0.47\textwidth}
    \centering
    \vspace{-7mm}
    \includegraphics[width=0.475\textwidth]{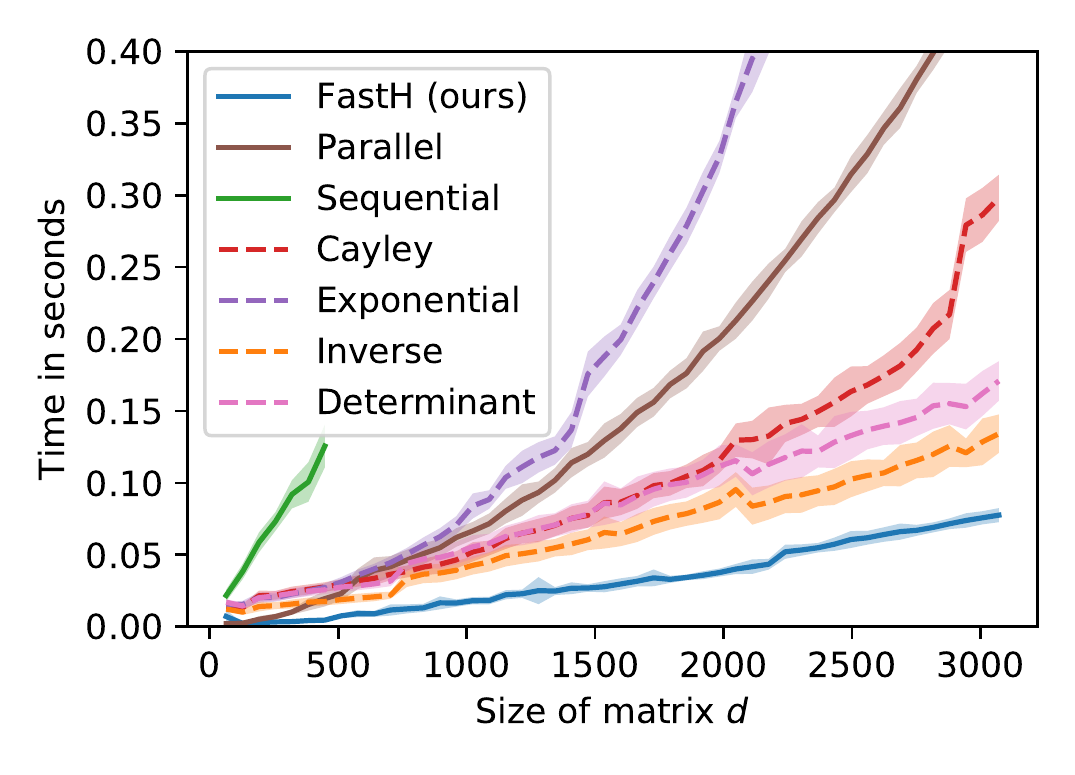}
    \caption{Running time of matrix operations. 
    Solid lines depict approaches which use the SVD reparameterization and dashed lines depict standard methods like \textsc{torch.inverse}. 
    }
    \label{fig:matrix_operations}
\end{wrapfigure}

\paragraph{Timing the Operations.} 
The matrix operations are usually used during the forward pass of a Neural Network, which change the subsequent gradient computations.
We therefore measure the sum of the time it takes to compute the matrix operation, the forward pass and the subsequent gradient computations. 

For example, with matrix inversion, we measure the time it takes to compute the matrix operation $\Sigma^{-1}$, the forward pass $W^{-1}X=V\Sigma^{-1}U^TX$ and the subsequent gradient computation wrt. $U, \Sigma, V$ and $X$. 
The measured time is then compared with \textsc{torch.inverse}, i.e, we compare against the total time it takes to compute \textsc{torch.inverse}(W), the forward pass $W^{-1}X$, and the subsequent gradient computation wrt. $W$ and $X$.

\paragraph{Setup.} 
We run the SVD reparameterization with three different algorithms: FastH, the sequential and the parallel algorithm from \cite{svdnn}. 
For each matrix operation, we consider matrices $V,\Sigma,U,W\in \R^{d\times d}$ and $X\in \R^{d\times M}$, where $m=32$ and $d=1\cdot 64,2\cdot 64, ..., 48\cdot 64$. 
We repeat the experiment $100$ times, and report the mean time $\mu$ with error bars $[\mu-\sigma, \mu+\sigma]$ where $\sigma$ is the standard deviation of the running times over the $100$ repetitions. 
To avoid clutter, we plot only the time of FastH for the matrix operation it is slowest to compute, 
and the time of the sequential and parallel algorithms for the matrix operation they were fastest to compute.

\Cref{fig:matrix_operations} depicts the measured running time on the y-axis with the size of the $d\times d$ matrices increasing on the x-axis. 
Each matrix operation computed by a standard method is plotted as a dashed line, and the different algorithms for the SVD reparameterization are plotted as solid lines. 
In all cases, FastH is faster than the standard methods. For example, with $d=768$, FastH is $3.1\times$ faster than the Cayley map, $4.1\times$ faster than the matrix exponential, $2.7\times$ faster than inverse and $3.5\times$ faster than matrix determinant. 
The sequential algorithm is not fast enough to speed up any matrix operation.

\section{Related Work}  
\label{sec:rw}

\paragraph{The Householder Decomposition. }
The Householder decomposition of orthogonal matrices has been used in much previous works, e.g., 
\cite{hhnf,orthhh,svdnn,sylvesternf,emerging}. 
Previous work typically use a type of sequential algorithm that performs $O(d)$ sequential inner products. 
To circumvent the resulting long computation time on GPUs, previous work often suggest limiting the number of Householder matrices, which limits the expressiveness of the orthogonal matrix, introducing a trade-off between computation time and expressiveness. 

FastH takes the same asymptotic time as the sequential algorithm, however, it performs less sequential matrix operations, making it up to $27\times $ faster in practice.
Since FastH computes the same output as the previous sequential algorithms, 
it can be used in previous work without degrading the performance of their model.
In particular, FastH can be used to either (1) increase expressiveness at no additional computational cost or (2) retain the same level of expresiveness at lower computational cost. 

\paragraph{SVDs in Neural Networks.}
The authors of \cite{svdnn} introduced a technique that provides access to the SVD of the weights in a Neural Network without computing the SVD. 
Their motivation for developing this technique was the exploding/vanishing gradient issue in RNNs. 
In particular, they use the SVD reparameterization to force all singular values to be within the range $[1\pm \epsilon]$ for some small $\epsilon$. 

We point out that although their technique, in theory, can be used to speed up matrix operations, their algorithms are too slow to speed-up most matrix operations in practice. 
To mitigate this problem, we introduce a new algorithm that is more suitable for GPUs, which allows us to speed up several matrix operations in practice. 

\paragraph{Different Orthogonal Parameterizations. }
The SVD reparameterization by \cite{svdnn} uses the Householder decomposition to perform gradient descent with orthogonal matrices. 
Their work was followed by \cite{cayley1} that raises a theoretical concern about the use of Householder decomposition. 
Alternative approaches based on the matrix exponential and the Cayley map have desirable provable guarantees, which currently, it is not known whether the Householder decomposition possesses. 
This might make it desirable to use the matrix exponential or the Cayley map together with the SVD reparameterization from \cite{svdnn}. 
However, previous work spend $O(d^3)$ time to compute or approximate the matrix exponential and the Cayley map. 
These approaches are therefore undesirable, because they share the $O(d^3)$ time complexity with SVD and thus cannot speed up SVD computations. 

\paragraph{Normalizing Flows. }
Normalizing Flows \cite{nice} is a type of generative model that, in some cases \cite{glow,emerging}, entails the computation of matrix determinant and matrix inversion. 
\cite{glow} propose to use the PLU decomposition $W=PLU$ where $P$ is a permutation matrix and $L,U$ are lower and upper triangular. 
The decomposition allows the determinant computation in $O(d)$ time instead of $O(d^3)$. 
\cite{emerging} point out that a fixed permutation matrix $P$ limits flexibility. 
To fix this issue, they suggest using the $QR$ decomposition where $R$ is a rectangular matrix and $Q$ is orthogonal.
They suggest making $Q$ orthogonal by using the Householder decomposition which FastH can speed up. 
Alternatively, one could use the SVD decomposition instead of the QR or PLU decomposition. 

\section{Code} 
To make FastH widely accessible, we wrote a PyTorch implementation of the SVD reparameterization which uses the FastH algorithm. 
The implementation can be used by changing just a single line of code, i.e, change \textsc{nn.Linear} to \textsc{LinearSVD}. 
While implementing FastH, we found that Python did not provide an adequate level of parallelization. 
We therefore implemented FastH in CUDA to fully utilize the parallel capabilities of GPUs. 
Code: \url{www.github.com/AlexanderMath/fasth/}. 


\section{Conclusion}
\label{sec:con}
We pointed out that, in theory, the techniques from \cite{svdnn,orthhh} allow for decreasing the time complexity of several matrix operations used in Neural Networks.
However, in practice, we demonstrated that the techniques are not fast enough on GPUs for moderately sized use-cases. 
We proposed a novel algorithm, FastH, that remedies the issues with both algorithms from \cite{svdnn}, which is up to $27\times$ faster than the previous sequential algorithm. 
FastH introduces no loss of quality, it computes the same result as the previous algorithms, just faster. FastH brings two immediate benefits: (1) improves upon the techniques from \cite{svdnn} in such a way that it is possible to speed up matrix operations, and (2) speeds up previous work that employ the Householder decomposition as done in, e.g., \cite{hhnf,sylvesternf,emerging}. 

\clearpage 
\section*{Broader Impact}
\label{sec:broader-impact}
Our algorithm speeds up the use of Householder decompositions in Neural Networks. 
This can positively impact researchers who use Householder decompositions, since they will be able to execute experiments faster. 
This is particularly beneficial for researchers with a constraint on their computational budget, in other words, a PhD student with one GPU stands to benefit more than a lab with state-of-the-art GPU computing infrastructure. 
The reduction in computing time also decrease power consumption and thus carbon emissions. 
However, as a potential negative impact, it is possible that the decrease in computation time will increase the usage of Neural Networks and thus increase overall carbon emission. 

\bibliography{bibliography}

\begin{thebibliography}{10}

\bibitem{wydec}
Christian Bischof and Charles Van~Loan.
\newblock The {WY} {R}epresentation for {P}roducts of {H}ouseholder {M}atrices.
\newblock {\em SIAM Journal on Scientific and Statistical Computing}, 1987.

\bibitem{matrixexp2}
Mario~Lezcano Casado.
\newblock Trivializations for {G}radient-{B}ased {O}ptimization on {M}anifolds.
\newblock In {\em NeurIPS}, 2019.

\bibitem{nice}
Laurent Dinh, David Krueger, and Yoshua Bengio.
\newblock {NICE:} {N}on-{L}inear {I}ndependent {C}omponents {E}stimation.
\newblock In {\em {ICLR} (Workshop)}, 2015.

\bibitem{cayley1}
Adam Golinski, Mario Lezcano-Casado, and Tom Rainforth.
\newblock Improving {N}ormalizing {F}lows via {B}etter {O}rthogonal
  {P}arameterizations.
\newblock In {\em ICML Workshop on Invertible Neural Networks and Normalizing
  Flows}, 2019.

\bibitem{revnet}
Aidan~N Gomez, Mengye Ren, Raquel Urtasun, and Roger~B Grosse.
\newblock The {R}eversible {R}esidual {N}etwork: {B}ackpropagation {W}ithout
  {S}toring {A}ctivations.
\newblock In {\em NIPS}, 2017.

\bibitem{emerging}
Emiel Hoogeboom, Rianne van~den Berg, and Max Welling.
\newblock Emerging {C}onvolutions for {G}enerative {N}ormalizing {F}lows.
\newblock In {\em ICML}, 2019.

\bibitem{glow}
Diederik~P Kingma and Prafulla Dhariwal.
\newblock Glow: Generative {F}low with {I}nvertible 1x1 {C}onvolutions.
\newblock In {\em NeurIPS}. 2018.

\bibitem{matrixexp1}
Mario Lezcano-Casado and David Mart\'{\i}nez-Rubio.
\newblock Cheap {O}rthogonal {C}onstraints in {N}eural {N}etworks: A {S}imple
  {P}arametrization of the {O}rthogonal and {U}nitary {G}roup.
\newblock In {\em ICML}, 2019.

\bibitem{cayley2}
Jun Li, Fuxin Li, and Sinisa Todorovic.
\newblock Efficient {R}iemannian {O}ptimization on the {S}tiefel {M}anifold via
  the {C}ayley {T}ransform.
\newblock In {\em ICLR}, 2020.

\bibitem{orthhh}
Zakaria Mhammedi, Andrew Hellicar, Ashfaqur Rahman, and James Bailey.
\newblock Efficient {O}rthogonal {P}arametrisation of {R}ecurrent {N}eural
  {N}etworks {U}sing {H}ouseholder {R}eflections.
\newblock In {\em ICML}, 2017.

\bibitem{sngan}
Takeru Miyato, Toshiki Kataoka, Masanori Koyama, and Yuichi Yoshida.
\newblock Spectral {N}ormalization for {G}enerative {A}dversarial {N}etworks.
\newblock In {\em ICLR}, 2018.

\bibitem{strang}
Gilbert Strang.
\newblock {\em Linear {A}lgebra and its {A}pplications}.
\newblock 2006.

\bibitem{hhnf}
Jakub~M Tomczak and Max Welling.
\newblock Improving {V}ariational {A}uto-{E}ncoders using {H}ouseholder {F}low.
\newblock {\em arXiv preprint}, 2016.

\bibitem{qrhh}
Frank Uhlig.
\newblock Constructive {W}ays for {G}enerating ({G}eneralized) {R}eal
  {O}rthogonal {M}atrices as {P}roducts of ({G}eneralized) {S}ymmetries.
\newblock {\em Linear Algebra and its Applications}, 2001.

\bibitem{sylvesternf}
Rianne van~den Berg, Leonard Hasenclever, Jakub Tomczak, and Max Welling.
\newblock Sylvester {N}ormalizing {F}lows for {V}ariational {I}nference.
\newblock In {\em UAI}, 2018.

\bibitem{compresswithsvd}
Jian Xue, Jinyu Li, and Yifan Gong.
\newblock Restructuring of {D}eep {N}eural {N}etwork {A}coustic {M}odels with
  {S}ingular {V}alue {D}ecomposition.
\newblock 2013.

\bibitem{svdnn}
Jiong Zhang, Qi~Lei, and Inderjit Dhillon.
\newblock Stabilizing {G}radients for {D}eep {N}eural {N}etworks via
  {E}fficient {SVD} {P}arameterization.
\newblock In {\em ICML}, 2018.

\end{thebibliography}
\bibliographystyle{plain}

\clearpage

\section{Supplementary Material}

\subsection{Proof of \Cref{thm:backwards}. }\label{sup:proof_backwards}

\begin{theorem*}
\Cref{algo:backward} computes $\frac{\partial L}{\partial X}$ and 
$\frac{\partial L}{\partial v_1},...,\frac{\partial L}{\partial v_d}$
in $O(d^2m)$ time with $O(d/m+m)$ sequential matrix~multiplications. 
\end{theorem*}

\begin{proof}
\textbf{Correctness. }FastH computes gradients by the same equations as \cite{svdnn}, so in most cases, we show correctness by clarifying how FastH computes the same thing, albeit faster. 

Consider $\frac{\partial L}{\partial X}$ computed in Step 1: 
\begin{align*}
    \frac{\partial L}{\partial X}= \frac{\partial L}{\partial A_{d/m+1}}&=P_{d/m}^T \cdots P_{1}^T\frac{\partial L}{\partial A_1} \\
    &=H_d^T \cdots H_1^T \frac{\partial L}{\partial A_1}. &&\cref{equ:P}
\end{align*}
This is the same as that computed in \cite{svdnn}. 

Consider Step 2. 
Both $\frac{\partial L}{\partial \widehat v_j}$ and $\frac{\partial L}{\partial \widehat A_j}$ are computed as done in \cite{svdnn}. 
$\widehat A_{j+1}$ is computed using \Cref{equ:A_hat} similar to backpropagation without storing activations \cite{revnet}, but using the fact that $\widehat H_j^T = \widehat H_j^{-1}$.

\textbf{Time Complexity. } 
In Step 1, the for loop performs $d/m$ matrix multiplications. 
Due to the WY decomposition $P_i^T=(I-2WY^T)^T=I-2YW^T$ which can be multiplied on $\frac{\partial L}{\partial A_i}\in \R^{d\times m}$ in $O(dm^2)$ time since $W,Y\in \R^{d \times m}$. 
The computation is repeated $d/m$ times, and take a total of $O(d^2m)$ time. 

Step 2 line 12 in \Cref{algo:backward_} performs two Householder matrix multiplications which take $O(dm)$ time, see \Cref{equ:A_hat}. 
In line 13, the gradient is computed by \Cref{equ:back_v}, this sum also takes $O(dm)$ time to compute. 
Both computations on line 12 and 13 are repeated $d/m \cdot m$ times, see line 8 and line 11. 
Therefore, the total time is $O(d^2m)$.

\textbf{Number of Sequential Operations. } Step 1 performs $O(d/m)$ sequential matrix operations.  
Lines 11-14 of Step 2 perform $O(m)$ sequential matrix multiplications. 
Since each iteration of line 8-15 is run in parallel, the algorithm performs no more than $O(d/m+m)$ sequential matrix multiplications. 
\end{proof}

\begin{algorithm}[H]
    \caption{FastH Backward}
    \label{algo:backward_}
    \footnotesize
    \begin{algorithmic}[1]
       \STATE {\bfseries Input:}  $A_1,...,A_{d/m+1}$, $P_1,..., P_{d/m}$ and $\frac{\partial L}{\partial A_1}$. 
       \STATE {\bfseries Output:} 
       $\frac{\partial L}{\partial X}$ and 
       $\frac{\partial L}{\partial v_k}$ for all $k$  where 
       $H_k=I-2\frac{v_kv_k^T}{||v_k||^2_2}$.
       \STATE // Step 1 
       \FOR[\textbf{sequentially}]{$i=1$ {\bfseries to} $d/m$}
       \STATE $\frac{ \partial L}{\partial A_{i+1}}=P_i^T\frac{\partial L}{\partial A_i}$ \cref{equ:back_A}.  \COMMENT{\hfill $\triangleright \; O(dm^2)$}
       \ENDFOR
       \vspace{2mm}
       \STATE // Step 2 
       \FOR[\textbf{in parallel}]{$i=1$ {\bfseries to} $d/m$}
    	\STATE Let $\frac{\partial L}{\partial \widehat{A}_1}=\left(\frac{\partial L} {\partial A_i}\right)$. 
        \STATE To ease notation, let $P_i= \widehat H_1 \cdots \widehat H_m$.
       \FOR{$j=1$ {\bfseries to} $m$}
    	\STATE Compute $\widehat A_{j+1}, \frac{\partial L}{\partial \widehat A_{j}}$
    	see \cref{equ:A_hat}.   \COMMENT{\hfill $\triangleright \; O(dm)$} \\
    	\STATE Compute $\frac{\partial L}{\partial \widehat v_j}$ 
    	using $\widehat A_{j+1},\frac{\partial L}{\partial \widehat A_j}$, 
    	\cref{equ:back_v}. 
    	\COMMENT{\hfill $\triangleright \; O(dm)$}
       \ENDFOR
       \ENDFOR
       \STATE \textbf{return} $\frac{\partial L}{\partial X} =\frac{\partial L}{\partial A_{d/m+1}}$ and $\frac{\partial L}{\partial v_k}$ 
       for all $k=1,...,d$. 
    \end{algorithmic}
\end{algorithm}

\subsection{Comparing Running Time }\label{sup:running_time}

This subsection clarifies how the matrix exponential and the Cayley map was used in combination with the SVD reparameterization from \cite{svdnn}. 
It also provides further details on the exact computations we timed in the experiment. 
These details were left out of the main article as they require the introduction of some notation regarding a \emph{reparameterization function}. 

Let $V\in \R^{d\times d}$ be a weight matrix and let $\phi$ be a function that reparameterizes $V$ so $\phi(V)$ is orthogonal, and we can perform gradient descent wrt. $V$. 
The Householder decomposition can be used to construct such a function $\phi$, by letting the columns of $V$ be Householder vectors and $\phi(V)$ be the product of the resulting Householder matrices. 

There exist alternative ways of constructing $\phi$ which does not rely on the Householder decomposition.
For example, the matrix exponential approach where $\phi_{exp}(V)=e^V$ and the Cayley map approach where $\phi_C(V)=(I-V)(I+V)^{-1}$ \cite{matrixexp2}. 

We record the joint time it takes to compute $\phi(V)X$ and the gradients wrt. $V$ and $X$ for a dummy input $X\in\R^{d\times M}$.  
To simplify the gradient computation of $V$, we use a dummy gradient $G\in \R^{d\times M}$ st. the gradient wrt. $V$ is $[\frac{ \partial \phi(V)\cdot X}{\partial V}]^T G $. 
It might be useful to think of $G$ as the gradient that arises by back-propagating through a Neural Network. 

Both the dummy input and the dummy gradient have normally distributed entries $X_{ij},G_{ij}\sim N(0,1)$.

\paragraph{Implementation Details. }
The parallel algorithm from \cite{svdnn} halted for larger values of $d$. 
The failing code was not part of the main computation. 
This allowed us to remove the failing code and still get a good estimate of the running time of the parallel algorithm. 
We emphasize that removing the corresponding code makes the parallel algorithm faster. 
The experiments thus demonstrated that FastH is faster than a lower bound on the running time of the parallel algorithm. 

\subsection{Using the SVD to Compute Matrix Operations}\label{sup:matrix_operations}

This section requires first reading \Cref{subsec:running_time} and \Cref{subsec:matrix_operations}. 
Recall that we, in \Cref{subsec:matrix_operations}, want to measure the total time it takes to compute both the matrix operation, the forward pass and the gradient computations. 
For example, with matrix inversion, we want to compute the matrix operation $\Sigma^{-1}$, the forward pass $V\Sigma^{-1}U^TX$ and the gradient computations wrt $V,\Sigma, U, X$. 

The time of the forward pass and gradient computations is no more than two multiplications and two gradient computations, which is exactly two times what we measured in \Cref{subsec:running_time}.
We re-used those measurements, and add the time it takes to compute the matrix operation, e.g., $\Sigma^{-1}$. 

\paragraph{Over Estimating the Time of FastH.} The matrix exponential and the Cayley map require one orthogonal matrix instead of two, i.e., $U\Sigma U^T$ instead of $U\Sigma V^T$. 
The WY decomposition then only needs to be computed for $U$ and not both $U$ and $V$. 
By re-using the data, we measure the time of two orthogonal matrices, this thus estimates an upper-bound of the real running time of FastH.

\end{document}